%% file: adaptive_3pc.tex
\documentclass[nohyperref]{article}

\usepackage{microtype}
\usepackage{graphicx}
\usepackage{subfigure}
\usepackage{booktabs} 
\usepackage{algorithm}
\usepackage{algorithmic}

\usepackage[pagebackref=true]{hyperref}

\usepackage[final]{neurips_2022_arxiv}

\usepackage{amsmath}
\usepackage{amssymb}
\usepackage{mathtools}
\usepackage{amsthm}
\usepackage{diagbox}
\usepackage[capitalize,noabbrev]{cleveref}
\usepackage{makecell}

\input{preamble.tex}

\input{new_commands} 
\input{theorem_environments} 
\input{new_math_operators} 

\usepackage[colorinlistoftodos,bordercolor=orange,backgroundcolor=orange!20,linecolor=orange,textsize=scriptsize]{todonotes}

\renewcommand*\backref[1]{\ifx#1\relax \else (Cited on #1) \fi}

\title{\bf Adaptive Compression for Communication-Efficient Distributed Training}

\author{\bf Maksim Makarenko\\ KAUST  \and \bf Elnur Gasanov  \\ KAUST  \and \bf Rustem Islamov \\ Institut Polytechnique de Paris \and \bf Abdurakhmon Sadiev \\ MIPT \and   \bf Peter Richt\'{a}rik \\KAUST}

\date{}
\begin{document}
\maketitle

\begin{abstract}
	We propose Adaptive Compressed Gradient Descent (\algname{AdaCGD}) -- a novel optimization algorithm for communication-efficient training of supervised machine learning models with adaptive compression level. Our approach is inspired by the recently proposed three point compressor (\algname{3PC}) framework of \citet{3PC}, which includes error feedback (\algname{EF21}), lazily aggregated gradient (\algname{LAG}), and their combination as special cases, and offers the current state-of-the-art rates for these methods under weak assumptions. While the above mechanisms offer a fixed compression level, or adapt between two extremes only, our proposal is to perform a much finer adaptation. In particular, we allow the user to choose any number of arbitrarily chosen contractive compression mechanisms, such as Top-$K$ sparsification with a user-defined selection of sparsification levels $K$, or quantization with a user-defined selection of quantization levels, or their combination. \algname{AdaCGD} chooses the appropriate compressor and compression level adaptively during the optimization process. Besides i) proposing a theoretically-grounded multi-adaptive communication compression mechanism, we further ii) extend the 3PC framework to bidirectional compression, i.e., we allow the server to compress as well, and iii) provide sharp convergence bounds in the strongly convex, convex and nonconvex settings. The convex regime results are new even for several key special cases of our general mechanism, including \algname{3PC} and \algname{EF21}. In all regimes, our rates are superior compared to all existing adaptive compression methods.
\end{abstract}


\section{Introduction}

Training machine learning models is computationally expensive and time-consuming. In the recent years, researchers have tended to use increasing datasets, often distributed over several devices, throughout the training process in order to improve the effective generalization ability of contemporary machine learning frameworks\citep{vaswani2019fast}. By word ``device'' or ``node'' we refer to any gadget that can store data, compute loss values and gradients (or stochastic gradients), and communicate with other different devices. For example, this distributed setting appears in {\em federated learning} (FL) \citep{konevcny2016federated, mcmahan2017communication, lin2018deep}. FL describes machine learning in a setting where a substantial number of strongly heterogeneous clients attempt to cooperatively train a model using the varied data stored on these devices without violating clients' information privacy\citep{3PC}. In this setting, distributed methods can be very efficient\citep{goyal2017accurate, you2019large} and therefore federated frameworks have attracted tremendous attention in recent years.

Dealing with the distributed environment, we consider optimization problem of the form
\begin{equation}\label{eq:finit_sum}
	\squeeze \min \limits_{x \in \mathbb{R}^{d}}\left\{f(x):=\frac{1}{n} \sum \limits_{i=1}^{n} f_{i}(x)\right\},
\end{equation}
where $x \in \mathbb{R}^{d}$ is the parameter vector of training model, $d$ is the dimensionality of the problem (number of trainable features), $n$ is the number of workers/devices/nodes, and $f_{i}(x)$ is the loss incurred by model $x$ on data stored on worker $i$. The loss function $f_{i}: \mathbb{R}^{d} \rightarrow \mathbb{R}$ often has the form of expectation of some random function $f_{i}(x):=\mathbb{E}_{\xi \sim \mathcal{D}_{i}}\left[f_{\xi}(x)\right]$ with $\mathcal{D}_{i}$ being the distribution of training data owned by worker $i$. 
In federated learning, these distributions are allowed to be different (so-called \emph {heterogeneous} case).
This finite sum function form allows us to capture the distributed nature of the problem in a very efficient way.

\begin{algorithm}[!t]
	\caption{\algname{DCGD} method with master compression}
	\label{alg:DCGD_master}
	\begin{algorithmic}[1]
		\STATE \textbf{Input:} starting point $x^0\in\R^d$; $g^0, \tilde{g}_i^0 \in \R^d$ for $i=1,\cdots, n$ (known by nodes), $\tilde{g}^0 = \frac{1}{n}\sum_{i=1}^n\tilde{g}_i^0$ (known by master); learning rate $\gamma > 0$, worker compressor $\cM^{\text{W}}$, master compressor $\cM^M$.
		\FOR{$t=0,1,2,\cdots, T-1$}
		\STATE Server broadcasts $g^t$ to all workers
		\FOR{{\bf  all devices} $i = 1, \dots, n$ {\bf in parallel}} 
		\STATE $x^{t+1} = x^t - \gamma g^t$
		\STATE $\tilde{g}_i^{t+1} = \cM^{\text{W}}(\nabla f_i(x^{t+1}))$
		\STATE Communicate $\tilde{g}_i^{t+1}$ to the server
		\ENDFOR
		\STATE  Server aggregates received gradient estimators $\tilde{g}^{t+1} = \frac{1}{n}\sum_{i=1}^n\tilde{g}_i^{t+1}$
		\STATE  $g^{t+1} = \cM^M(\tilde{g}^{t+1})$
		\ENDFOR
	\end{algorithmic}
\end{algorithm}

\subsection{Communication-efficient distributed learning via gradient compression}

The most effective models are frequently over-parameterized, which means that they contain more parameters than there are training data samples\citep{ACH-overparameterized-2018}. 

In this case, distributed methods may experience \textit{communication bottleneck}: the situation when the communication cost for the workers to transfer sensitive information in joint optimization can exceed by multiple orders of magnitude the cost of local computation\citep{dutta2020discrepancy}. One of the practical methods to transfer information more efficiently is to apply a local compression operator \citep{seide20141, suresh2017distributed, konevcny2018randomized} to the model's parameters (gradients or tensors) needs to be communicated across different clients. The compression operator could be formalized as (possibly randomized) mapping $\cC: \R^d \rightarrow \R^d$, where $d$ is the size of the tensor that has to be transmitted, with the feature that transmission of compressed tensor $\cC(x)$ requires much less communication cost than the transfer of initial tensor $x$. While compression decreases the number of bits that are transferred during each communication cycle, it also brings in errors. As a result, the number of rounds necessary to obtain a solution with the appropriate accuracy typically increases. However, as the trade-off frequently appears to favor compression over no compression, compression has been proven to be effective in practice. 

Distributed Compressed Gradient Descent (\algname{DCGD}) \citep{DCGD} provides the simplest and universal mechanism for distributed communication-efficient training with compression.  With the given learning rate $\gamma$, \algname{DCGD} implements the following update rule \begin{equation}\label{eq:DCGD} \squeeze x^{t+1} = x^t - \gamma \frac{1}{n}\sum \limits_{i=1}^n g_i^t, \quad g_i^t = \cM_i^t(\nabla f_i(x^t)).\end{equation}
Here, $g_i^t$ represents an estimated gradient, result of mapping of original dense and high-dimensional gradient $\nabla f_i(x^t) \in \R^d$ into a vector of same size that can be transferred efficiently with far fewer bits via $\cM_i^t$ compression mechanism. 

\subsection{\algname{DCGD} with bidirectional compression}

In some cases \citep{DoubleSqueeze, Artemis2020, EF21BW} it is desirable to add compression on the server side to have efficient communication between server and clients in both directions. One could easily extend the general framework of \algname{DCGD} to the case of bidirectional compression. If we define the general master compression mechanism as $\cM^{M,t}$ and worker compression mechanism as $\cM_i^{W,t}$ we could formally write the general bidirectional \algname{DCGD} algorithm as \Cref{alg:DCGD_master}.

\section{Motivation and Background}
\label{sec:mot_and_back}

The main motivation of this work is to generalize the ideas presented in \citep{3PC} to allow compression level evolve during the optimization process based on some local information about client's individual cost function.

\subsection{Constant contractive compressors}

The majority of methods employing gradient compression mechanisms use static compressors with constant compression level. In this approach\citep{3PC}, one sets
\begin{equation}\label{eq:b9fd-98ybhkfd}\cM_i^t(x) \equiv \cC(x),\end{equation}
where $\cC:\R^d \to \R^d$ is a (possibly randomized) operator. There are two large classes of operators (or compressors) that have been analyzed in the literature: i) \emph{unbiased} compression operators and ii) \emph{biased} or \emph{contractive} compression operators. In this work we deal with {\em contractive} compressors only. Here we explicitly give the definition of this class.

\begin{definition}[Biased or contractive compression operator]
	A mapping $\cC: \R^d \rightarrow \R^d$ is called \emph{biased} or \emph{contractive}  compression operator there exists $0<\alpha\leq 1$ such that 
	\begin{eqnarray}\label{eq:b_compressor}
		\Exp{\|\cC(x) - x\|^{2}} \leq \rb{1 - \alpha} \|x\|^{2}, \qquad \forall x\in \R^d.
	\end{eqnarray}
\end{definition}

Rank-$K$~\citep{DCGD} and Top-$K$~\citep{Alistarh-EF-NIPS2018} sparsification compressors are typical examples of contractive compressors. Due to the biased nature of these compressors, until recently, there was a gap between experimental and theoretical development of gradient descent methods based on contractive compressors. Thus, during the last years, algorithmic approaches have provided several methods of high practical importance, most notable of which is the so-called error feedback mechanism \citep{seide20141}, fixing a divergence issue that appeared in practice. In contrast, in the theoretical development, until very recently, analytical studies offered weak sublinear \citep{Stich-EF-NIPS2018,Karimireddy_SignSGD,A_better_alternative} convergence rates under, in some cases, strong unrealistic assumptions~\citep{3PC}. Recently,  \citet{ef21} fixed this by providing a novel algorithmic and theoretical development that recovers \algname{GD} $\cO(1/T)$ rates, with the analysis using standard assumptions only. \citet{EF21BW} subsequently extended the \algname{EF21} framework by including several algorithmic and theoretical extensions, such as bidirectional compression and client stochasticity, which makes this method of high practical interest. Despite these advances, there are still many  challenges in the theoretical understanding of these classes of methods. One of such challenges is a lack of precise theoretical study with the strong rates for error feedback methods in a convex setting. 

\subsection{Existing adaptive compressors}

\begin{table}
	\caption{Summary of adaptive compressed methods. $n$ is a number of workers, $L$ and $\mu$ are  smoothness and strong convexity constants respectively of $f_i \ \forall i \in \{1, \dots, n\}$, $\kappa = \frac{L}{\mu}$ is a condition number, $C_i$s are constants, $\Delta_x = \|x^0 - x^\ast\|^2, \Delta_f = f(x_0) - f^\ast$, $ M_1 = \max\{L_{-} + L_{+}\sqrt{\frac{2B_{\max}}{A_{\min}}}, \frac{1}{A_{\min}} \},  M_2 = L_{-} + L_{+}\sqrt{\frac{B_{\max}}{A_{\min}}} $ (see~\Cref{lm:adaptive_3PC_is_3PC}). str cvx = strongly convex, cvx = convex, noncvx = nonconvex.}
	\label{table:comparison}
	\scriptsize
	\centering
	\begin{threeparttable}
		\setlength{\tabcolsep}{3pt}
		\begin{tabular}{llllll}
			\toprule
			Paper     & Any $\cC$? & Theory?& \thead[l]{Str cvx / \\ P\L\  noncvx \\rate} & Cvx rate& \thead[l]{General\\ noncvx \\rate}\\
			\midrule
			\cite{adaptive_quantization_model_updates_fl} & \xmark  & \xmark & \xmark & \xmark& \xmark\\
			\cite{Abdelmoniem.EuroMLSys21} & \xmark  & \xmark & \xmark& \xmark& \xmark\\
			\cite{dadaquant} & \xmark  & \cmark \tnote{(\color{blue}1)} & $\frac{\max\{\kappa, \frac{\kappa^2}{n}, \frac{n}{\mu^2} \}}{T^2}$ & \xmark & $\cO(\frac{L\Delta_f}{\sqrt{T}}  + \frac{C_1}{T})$ \\
			\cite{FedDQ} & \xmark & \cmark & \xmark & \xmark& $\cO(\frac{L\Delta_f}{\sqrt{T}})$\tnote{(\color{blue} 2)}\\
			\cite{zhao2022aquila} & \xmark & \cmark & linear \tnote{(\color{blue} 3)} &\xmark & \xmark\\
			\cite{mao_adaptive_quantization} & \xmark  & \cmark & linear \tnote{(\color{blue} 3)} & \xmark& \xmark\\
			\cite{khirirat_flexible} & \cmark  & \xmark & \xmark& \xmark& \xmark\\
			\cite{mishchenko2022intsgd} & \xmark & \cmark\tnote{(\color{blue} 4)}& \xmark& $\cO(\frac{L \Delta_x}{T} + \frac{ \sigma^2_\ast + \varepsilon}{Ln})$&$\cO(\frac{L\Delta_f}{T} + \frac{\varepsilon}{Ln})$ \\
			THIS WORK & \cmark  & \cmark & $ \left(1 - \min\left\{\frac{\mu}{M_2}, A_{\min}\right\}\right)^T $ &$\cO\left(\frac{M_1}{T}\right) $ &$ \cO\left(\frac{2\Delta_f M_2 + C_3 / A_{\min}}{T}\right)$ \\
			\bottomrule
		\end{tabular}
		\begin{tablenotes}
			\item [\color{blue}(1)] The rates, as stated in the paper, are derived from~\cite{fedpaq}. We consider non-local full gradient setup, \textit{i.e.} $\sigma^2 = 0$ and $\tau = 1$.
			\item [\color{blue}(2)] We show the rate for non-local full gradient setup, \textit{i.e.} $\sigma^2 = 0$ and $\tau = 1$.
			\item [\color{blue}(3)] Their work does not present any \textit{explicit} rate.
			\item [\color{blue} (4)] $\varepsilon > 0$ is a parameter of IntSGD algorithm.
		\end{tablenotes}
	\end{threeparttable}
\end{table}

Using a static compression level of the compressor for all clients could limitate the optimization framework's capability. Indeed, compression in the FL scenario can depend on the client it is applied on. For example, in heterogeneous hardware cases, \textit{i.e.} when machines participating in collaborative training have very different computational capabilities, adjusting the compression level of a compressor for every client could help to reduce overall training time~\citep{fjord,Abdelmoniem.EuroMLSys21}. Ideally, the optimizer should be able to define the particular compression level for each client separately based on the local information from the client.

Despite the significant practical interest in the development of such methods, there is currently very limited research and understanding of adaptive mechanisms of this type. Only a few papers \citep{FedDQ,dadaquant, mishchenko2022intsgd} provide any convergence guarantees with explicit rates, and most of them are designed for the specific type of compressors only, mostly quantizers. So, in \citep{adaptive_quantization_model_updates_fl}, the authors design a mechanism for adaptive change of quantization level $s^k \sim \sqrt{\frac{f(x^0)}{f(x^k)}}$ of a random dithering operator~\citep{alistarh2017qsgd}. DAdaQuant~\citep{dadaquant} and FedDQ~\citep{FedDQ} suggest ascending and descending quantizations throughout the training. AQUILA~\citep{zhao2022aquila} and AGQ~\citep{mao_adaptive_quantization} build an adaptive quantization on top of the Lazily Aggregated Quantized (LAQ) gradient algorithm~\citep{LAQ}. IntSGD~\citep{mishchenko2022intsgd} adaptively sets the scaling parameter $\alpha^k$ of a vector plugged in a randomized integer rounding operator. CAT S+Q ~\citep{khirirat_flexible} proposes an adaptive way to choose $k$: the top-$k$ elements of the gradient at iteration $i$, only \textit{signs} of which clients send to the server along with the gradient norm. \Cref{table:comparison} provides a detailed comparison of these works. 

\subsection{Adaptive compression via selective (\emph{lazy}) aggregation}

The \algname{LAG} mechanism proposed by \citet{LAG} is an alternative way to embed adaptivity into the framework by introducing communication "skipping". According to the lazy aggregation communication mechanism, each worker $i$ only shares its local gradient if it is significantly different from the last gradient shared previously. Otherwise, the worker decides to "skip" the communication round. In some sense, it is an adaptive mechanism that chooses between two extremes for each client: sending a full gradient or skipping the communication round. 

\citet{3PC} recently generalized this idea by introducing \algname{CLAG}, which connects particular contractive $C$ compressor with a pre-defined compression level with \algname{LAG} mechanism. In \citeauthor{3PC}'s \algname{CLAG} method all $n$ workers send the compressed gradient $g_i^0 = \cC(\nabla f_i(x^0))$ for all $i\in [n]$, at the beginning of the training. The workers $i\in [n]$ define $g_i^{t+1}$, which can be viewed as a shifted and compressed version of the client's gradient $\nabla f_i(x^{t+1})$ using the {\em lazy aggregation rule} combined with \algname{EF21} shift
\begin{equation}\label{eq:clag_rule}
	g_i^{t+1} = \begin{cases} g_i^t + \cC\left(\nabla f_i(x^{t+1}) - g_i^t\right),& \text{if } \|\nabla f_i(x^{t+1}) - g_i^t\|^2 > \zeta D_i^t,\\ g_i^t,& \text{otherwise} \end{cases}
\end{equation}
where $D_i^t \eqdef \|\nabla f_i(x^{t+1}) - \nabla f_i(x^t)\|^2$ and $\zeta>0$ is the {\em trigger}. Trigger parameter $\zeta$ controls how frequently trigger condition will be satisfied and how often clients skip communication rounds. \algname{CLAG} method includes \algname{LAG} as a special case with $\cC$ compressor being identity operator ({\em no compression}). This approach also could be seen as adaptive which interpolates between two extremes: compressed gradient with pre-defined compression level or entirely skipping communication.  

Although both \algname{LAG} and \algname{CLAG} perform well in practice, their fixed and limited compression levels could restrict their performance and make these methods sub-optimal. It is of particular practical interest to create a more general method with evolving fine-tuned compression level individual for every client. From the perspective of the convergence theory, one of the issues  {\em lazy} methods have is the difficulty of determining how often communication skips occur because the trigger is conditional. Thus, there are no theoretical studies examining the frequency of communication skips.

\section{Summary of Contributions}
\label{sec:contributions}


\begin{table*}[t]
	\centering
	\scriptsize
	\caption{Comparison of available convergence guarantee results of methods employing lazy aggregation.}
	\label{tab:rates}    
	\begin{threeparttable}
		\begin{tabular}{| l | l l l l l l|}
			\hline
			Method & \thead[l]{Adaptive \\ compression?} &\thead[l]{Bidirectional\\ case}& \thead[l]{Str cvx \\ case} & Cvx case & \thead[l]{{P\L} \\ noncvx  case} &  \thead[l]{General \\noncvx case} \\ 
			\hline
			\hline								
			\algname{\tiny LAQ}  \citep{LAQ} & \xmark   & \xmark  & \cmark & \xmark & \xmark  & \xmark \\ 
			\algname{\tiny LENA}  \citep{LENA}  & \xmark  & \xmark  & \cmark& \cmark & \cmark & \cmark \\
			\algname{\tiny LAG}  \citep{3PC}       & \xmark & \xmark & \cmark & \xmark & \cmark &\cmark \\
			\algname{\tiny CLAG}  \citep{3PC}     & \xmark & \xmark  & \cmark & \xmark &\cmark   & \cmark \\ 
			\hline				
			\algname{\tiny AdaCGD}  (NEW, 2022)       & \cmark & \cmark & \cmark   & \cmark  & \cmark & \cmark \\ 	
			\hline
		\end{tabular}
	\end{threeparttable}
\end{table*}

We highlight our main contributions as follows:

{$\bullet$ \bf New class of adaptive compressors.} In \cite{3PC}, the authors propose the different classes of compressors unified through a single theory. Despite the large variability of the compression mechanisms, including the algorithms with \emph{lazy} aggregation rule, the compression level in all of the considered algorithms is pre-defined before and kept constant during the training. In this work, we take a step further and formulate an extended class of an adaptive \algname{3PC} compressors (\algname{ada3PC}) with tunable compression levels defined by some general trigger rules. As an original \algname{3PC} compressors, this class of compressors are very general and includes a number of specific compressors such as \algname{AdaCGD} which includes \algname{EF21} and \algname{CLAG} as special cases. This method is applicable for a large class of biased compressors, such as Top-$K$ and Rank-$K$ sparsification. 

{$\bullet$ \bf Convergence guarantees with strong rates unified by a single \algname{3PC} theory.} We provide a strong convergence bound for strongly convex, convex, and non-convex settings. Comparing with the adaptive methods outside the \algname{3PC} context, we provide a more elaborate theory with better convergence rates. For \algname{AdaCGD} we recover $\cO(1/T)$ rate of \algname{GD} up to a certain constant in general non-convex case. It is superior in comparison with $\cO(1/\sqrt{T})$ rate \citep{dadaquant, FedDQ} for SOTA in adaptive compression outside \algname{3PC} context. The convergence theory in a convex set is of particular interest since due to its novelty even in the case of \algname{3PC} for some key cases of \algname{AdaCGD}, such as \algname{EF21} and \algname{CLAG}. In other words, it is a new SOTA result for the error-feedback method in the convex setting. 

{$\bullet$ \bf Extension of \algname{3PC} theory with bidirectional compression.} We extend \algname{3PC} methods with bidirectional compression i.e., we allow the server to compress as well. 

\Cref{tab:rates} compares \algname{AdaCGD} with other described in the literature lazy algorithms. 


\section{Ada3PC: A Compression-Adaptive 3PC Method}\label{sec:ada3pc}



\subsection{\algname{3PC} compressor}

\citet{3PC} introduces the general concept of three point compressors. Here we provide its formal definition for consistency:

\begin{definition}\label{def:3PC}
	We say that a (possibly randomized) map $\cC_{h, y}(x): \underbrace{\RR^d}_{h \in} \times \underbrace{\RR^d}_{y \in}  \times \underbrace{\RR^d}_{x \in} \rightarrow \RR^d$ is a three point compressor (3PC) if there exist constants $0 < A \leq 1$ and $B \geq 0$ such that the following relation holds for all $x, y, h\in \RR^d$
	\begin{equation}\label{ineq:3PC_def}
		\ExpBr{\|\cC_{h, y}(x) - x\|^2} \leq (1 - A) \|h - y\|^2 + B \|x - y\|^2.
	\end{equation}
\end{definition}

Authors show that \algname{EF21} compression mechanism satisfies \Cref{def:3PC}. Let $\cC:\R^d\to \R^d$ be a contractive compressor with contraction parameter $\alpha$, and define
\begin{eqnarray}
	\cC^{\operatorname{EF}}_{h,y}(x) \eqdef h + \cC(x - h).  \label{eq:---=08}
\end{eqnarray}
If we use this mapping to define a compression mechanism $\cM_i^t$ via \eqref{eq:DCGD} within \algname{DCGD}, we obtain the \algname{EF21} method.

Another variant of \algname{3PC} compressors introduced in \citep{3PC} is \algname{CLAG} compressor. Let $\cC:\R^d\to \R^d$ be a contractive compressor with contraction parameter $\alpha$. Choosing a trigger $\zeta>0$, authors define the \algname{CLAG} rule
\begin{eqnarray}
	\cC^{\operatorname{CL}}_{h,y}(x) \eqdef \begin{cases} h + \cC(x - h),& \text{if } \|x- h\|^2 > \zeta \|x - y\|^2,\\ h,& \text{otherwise,} \end{cases}  \label{eq:clag_98f9d8}
\end{eqnarray}
If we employ this mapping into \algname{DCGD} method \eqref{eq:DCGD} as communication mechanism $\cM_i^t$, we obtain \algname{CLAG}. It includes \algname{LAG} compressor $\cC^{\operatorname{L}}$ as a special case when compressor $\cC$ is identity.

\subsection{Adaptive 3PC compressor}
We are now ready to introduce the Adaptive Three Point (\algname{Ada3PC}) Compressor. 
\begin{definition}[\algname{Ada3PC} compressor]\label{def:adaptive_3PC}
	Let $\cC^1, \cC^2, \dots, \cC^m$ be 3PC compressors: $\cC^i: \RR^{3d} \rightarrow \RR^d$ for all $i$. Let $P_1, P_2, \dots, P_{m-1}$ be conditions depending on $h, y, x$, \textit{i.e.} $P_j:   \underbrace{\RR^d}_{h \in} \times \underbrace{\RR^d}_{y \in}  \times \underbrace{\RR^d}_{x \in} \rightarrow \{0, 1\}$ for~all~$j$. Then, the adaptive 3PC compressor, associated with above 3PC compressors and conditions, is defined as follows:
	\begin{equation}\label{eq:def_adaptive_3PC}
		\cC_{h, y}^{\text{ad}}(x) = \begin{cases}
			\cC_{h,y}^1 (x) &\text{ if } P_1(h, y, x),\\
			\cC_{h,y}^2(x) &\text{ else if } P_2(h, y, x), \\
			\dots, \\
			\cC_{h, y}^{m-1}(x) &\text{ else if } P_{m-1}(h, y, x),\\
			\cC_{h,y}^m(x) &\text{ otherwise.}
		\end{cases}
	\end{equation}
\end{definition}

\algname{Ada3PC} compressor first finds the smallest index $j$ for which $P_j(h, y, x) = 1$ (if such index does not exist, we set $j = m$). Then,~\algname{Ada3PC} applies $\cC^j$ compressor on vector $x$.

\subsection{Adaptive Compressed Gradient Descent}

There are many ways how to define meaningful and practical compressors in the context of the adaptive \algname{3PC} framework. Here we provide one particular, perhaps the simplest scheme, which we define as \algname{AdaCGD}. In this scheme we have a set of $m$ \algname{EF21} compressors \{$\cC^{\operatorname{EF},j}_{h,y}(x)$\}$_{j \in 1 \ldots m}$ sorted in order from the highest compression level to the lowest, i.e. $\alpha_1 \leq \alpha_2 \ldots \leq \alpha_m$, where $\alpha_j$ is a corresponding contractive parameter of the $j$-th compressor. For example, if we use Top-$K$ in $\cC^{\operatorname{EF}}_{h,y}$ compressors, first and last indices correspond to the compressors with the smallest and the largest $K$, respectively. We choose the first compressor, \textit{i.e.} with the strongest compression, which satisfies a trigger rule. We design the $j$-th trigger rule following an intuition of {\em lazy aggregation} rule:

\begin{equation}\label{eq:trigger_adaptive}
	P_j \eqdef \|x - \cC^{\operatorname{EF},j}_{h,y}(x)\|^2 \leq \zeta \|x - y\|^2.
\end{equation}

As in the original \algname{LAG} rule, the left side of \eqref{eq:trigger_adaptive} presents the difference between the true gradient and its estimate, while the right side compares gradient differences on neighboring iterations.The key difference of \eqref{eq:trigger_adaptive} trigger from \algname{LAG} and \algname{CLAG} rule \eqref{eq:clag_rule} is that the left side of this trigger condition depends explicitly from the level of compression. This feature is essential as it enables the desired rule-based compressor selection. Altogether, we can define \algname{AdaCGD} compressor formally.  
\begin{definition}[\algname{AdaCGD} compressor]
	Given the set of $m$ \algname{EF21} compressors \{$\cC^{\operatorname{EF},j}_{h,y}(x)$\}$_{j \in 1 \ldots m}$, sorted in descending order of compression level and $\zeta \geq 0$, adaptive \algname{AdaCGD} compressor is defined with a switch condition as follows:
	\begin{equation}
		\cC_{h, y}^{\text{AC}}(x) = \begin{cases}
			h &\text{ if } \|x - h\|^2 \leq \zeta \|x - y\|^2,\\
			\cC^{\operatorname{EF},1}_{h,y}(x) &\text{ else if } \|x - \cC^{\operatorname{EF},1}_{h,y}(x)\|^2 \leq \zeta \|x - y\|^2, \\
			\dots, \\
			\cC^{\operatorname{EF},m-1}_{h,y}(x) &\text{ else if } \|x - \cC^{\operatorname{EF},m-1}_{h,y}(x)\|^2 \leq \zeta \|x - y\|^2, \\
			\cC^{\operatorname{EF},m}_{h,y}(x) &\text{ otherwise.}
		\end{cases}
	\end{equation}
\end{definition}

If $\cC^{\operatorname{EF},m}_{h,y}$ uses Top-$d$ compression, \textit{i.e.}, identity operator, \algname{AdaCGD} is an adaptive compressor composed of the whole spectrum of compressors from full compression, i.e., communication "skip", to zero compression, i.e., sending full gradient. Since standalone "skip" connection is clearly not \algname{3PC} operator, it may not be obvious that \algname{AdaCGD} is a special case of \algname{Ada3PC}. For this reason, here we provide the following lemma:

\begin{lemma}\label{lm:adacgd_is_ada3pc}
	\algname{AdaCGD} is a special case of \algname{Ada3PC} compressor. 
\end{lemma}

It is easy to see that if $\zeta=0$ \algname{AdaCGD} reduces to \algname{EF21}. Similarly, \algname{CLAG} is a special case of \algname{AdaCGD} when $m = 1$.

\section{Theory}\label{sec:theory}
In this section, we present theoretical convergence guarantees for~\Cref{alg:DCGD_master} with \algname{Ada3PC} compressors in  two new cases presented in~\Cref{tab:rates}. The results for general and P\L \ nonconvex cases can be found in the appendix.

\subsection{Assumptions}\label{sec:as}

To get convergence rates of~\Cref{alg:DCGD_master}, we rely on standard assumptions. 

\begin{assumption}
	\label{as:convex}
	The functions $f_1,\dots, f_n: \R^d \rightarrow \R$ are convex, i.e.
	\begin{equation}
		\label{eq: convexity}
		f_i(x) - f_i(y) - \la\nabla f_i(y), x - y\ra \geq 0, \ \forall x, y \in \R^d, \forall i .
	\end{equation}
\end{assumption}

\begin{assumption}
	\label{as:lip_f}
	The function $f: \R^d \rightarrow \R$ is $L_{-}$-smooth, i.e.
	\begin{equation}
		\|\nabla f(x) - \nabla f(y)\|\leq L_{-}\|x - y\|, \ \forall x, y \in \R^d .
	\end{equation}
\end{assumption}

\begin{assumption}
	\label{as:lip_avr}
	There exists $L_{+} > 0$ such that $ \avein \|\nabla~f_i(x)- ~ \nabla f_i(y)\|^2\leq~L^2_{+}\|x~-~y\|^2$ for all $x, y \in \R^d$. Let $L_{+}$ be the smallest such number.
\end{assumption}

We borrow $\{L_{-}, L_{+}\}$ notation from~\citep{szlendak2022permutation}. \Cref{as:lip_avr} avoids a stronger assumption on Lipschitz smoothness of individual functions $f_i$. Moreover, one can easily prove that $L_{-} \leq L_{+}$.

The next assumption is standard for analysis of practical methods \citep{ahn2020sgd}, \cite{rajput2020closing}. However, compared to previous works, we require a more general version. 
\begin{assumption}\label{as:bounded_iterates}
	We assume  that there exists a constant $\Omega >0$ such that $\EE [\|x^t - x^\ast\|^2] \leq \Omega^2$, where $x^t$ is any iterate generated by~\Cref{alg:DCGD_master}.
\end{assumption}

\begin{assumption}\label{as:lower_bound}
	The functions $f_1, \dots, f_n$ are differentiable. Moreover, $f$ is bounded from below by an infimum $f^{\inf} \in \RR$, \textit{i.e.} $f(x) \geq f^{\inf}$ for all $x\in \RR^d$.
\end{assumption}

\subsection{Adaptive 3PC is a 3PC compressor}

While this may not be obvious at first glance, Adaptive 3PC compressor itself belongs to the class of 3PC compressors. We formalize this statement in the following lemma.

\begin{lemma}\label{lm:adaptive_3PC_is_3PC}
	Let $\cC^{\text{ad}}$ be an adaptive 3PC compressor. Let $\{\cC^i\}_{i=1}^m$ be 3PC compressors associated with $\cC^{\text{ad}}$, \textit{i.e.} for all $i$ there exists constants $0 < A_i \leq 1$ and $B_i \geq 0$, such that for all $h, y, x \in \RR^d$
	\begin{align*}
		\EE \|C^i_{h, y}(x) - x\|^2 \leq (1 - A_i) \|h - y\|^2 + B_i \|x - y\|^2.
	\end{align*}
	Then, $\cC^{\text{ad}}$ is a 3PC compressor with constants $A_{\min} \eqdef \min \{A_1, \dots, A_m\}$ and $B_{\max} \eqdef \max \{B_1, \dots, B_m\}$.
\end{lemma}

\begin{proof}
	Let us fix $h, y, x \in \RR^d$ and let $j$ be the index, such that $P_i(h, y, x) = 0$ for all $i < j$ and, if $j < m$,  $P_j(h, y, x) = 1$. Then, 
	\begin{align*}
		\EE \|\cC^{\text{ad}}_{h, y} (x) - x\|^2 = \EE \|\cC^j_{h, y} (x) - x\|^2 & \overset{\eqref{ineq:3PC_def}}{\leq} (1 - A_j) \|h - y\|^2 + B_j \|x - y\|^2\\
		& \leq (1 - A_{\min}) \|h - y\|^2 + B_{\max} \|x - y\|^2.
	\end{align*}
\end{proof}

In the definition of \algname{Ada3PC} compressor, we never specify what conditions $P_i$s are. In fact, they are completely arbitrary! This enables us to build infinitely many new compressors out of few notable examples, presented in~\citep{3PC}.

\subsection{Convergence}
In this subsection, we show how assumptions we make about minimized functions and compressors affect the convergence rate of~\Cref{alg:DCGD_master}.

\paragraph{Convergence for convex functions.} The first result assumes that $\cM^{\text{W}}$ in~\Cref{alg:DCGD_master} is a \algname{3PC} compressor, $\cM^{\text{M}}$ is an identity compressor: $\cM^{\text{M}}(x) = x~ \forall x \in \RR^d$. 

\begin{theorem}\label{thm:3PC_cvx}
	Let Assumptions~\ref{as:convex},~\ref{as:lip_f}, ~\ref{as:lip_avr} and~\ref{as:bounded_iterates} hold. In~\Cref{alg:DCGD_master}, assume $\cM^{\text{W}}$ is a \algname{3PC} compressor, $\cM^{\text{M}}$ is an identity compressor, and the stepsize $\gamma$ satisfies $0 < \gamma \leq 1/M$, where $M = L_{-} + L_{+}\sqrt{\frac{2B}{A}}$. Then, for any $T \geq 1 $ we have
	\begin{align*}
		\squeeze	\EE\left[f(x^T)\right] - f(x^{\star}) \leq  \max\left\{\frac{1}{\gamma}, \frac{1}{A}\right\}\frac{2(\Omega^2 +\Phi^0)}{T},	
	\end{align*}
	where  $\Phi^t \eqdef f(x^t) - f(x^{\star}) + \frac{\gamma}{A} \avein\left\|\nabla f_i(x^t) - g^t_i\right\|^2$ for any $t\geq 0$.
\end{theorem}

The theorem combined with~\Cref{lm:adaptive_3PC_is_3PC} implies the following fact.

\begin{corollary}
	Let the assumptions of~\Cref{thm:3PC_cvx} hold, assume $\cM^{\text{W}}$ is an \algname{Ada3PC} compressor, $\cM^{\text{M}}$ is an identity compressor, and choose the stepsize $\gamma = \frac{1}{L_{-} + L_{+} \sqrt{\frac{2B_{\max}}{A_{\min}}}}$. Then, for any $T \geq 1$ we have
	\begin{align*}
		\squeeze		\ExpBr{f(x^T)} - f(x^\ast) \leq  \max\left\{L_{-} + L_{+} \sqrt{\frac{2B_{\max}}{A_{\min}}}, \frac{1}{A_{\min}}\right\}\frac{2(\Omega^2 +\Phi^0)}{T}.	
	\end{align*}
	Thus, to achieve $	\ExpBr{f(x^T)} - f(x^\ast) \leq \varepsilon$ for some $\varepsilon >0$, the \algname{Ada3PC} method requires 
	\begin{align*}
		\squeeze		T =  \cO\left(\max\left\{L_{-} + L_{+}\sqrt{\frac{2B_{\max}}{A_{\min}}}, \frac{1}{A_{\min}} \right\} \frac{2(\Omega^2 + \Phi_0^2)}{\varepsilon}\right)
	\end{align*}
	iterations.
\end{corollary}

\paragraph{Convergence for bidirectional method.} Here, we analyze the case when meaningful compressors applied on both communication directions, \textit{i.e.}, both $\cM^{\text{M}}$ and $\cM^{\text{W}}$ are \algname{3PC} compressors.

\begin{theorem}\label{thm:nonconvex_bidir}
	Let Assumptions~\ref{as:lip_avr} and~\ref{as:lower_bound} hold. Let $\cM^{\text{M}}$ and $\cM^{\text{W}}$ be \algname{3PC} compressors and the stepsize in Algorithm~\ref{alg:DCGD_master} be set as
	\begin{equation}
		0 < \gamma \leq \frac{1}{L_{-} + L_{+}\sqrt{\frac{6B^{\text{M}}(B^{\text{W}}+1)}{A^{\text{M}}} + \frac{2B^{\text{W}}}{A^{\text{M}}}\left(1+\frac{3B^{\text{M}}(2-A^{\text{W}})}{A^{\text{M}}}\right)}}.
	\end{equation}
	Fix $T$ and let $\hat{x}^T$ be chosen uniformly from $\{x^0,x^1,\cdots,x^{T-1}\}$ uniformly at random. Then 
	\begin{equation}
		\squeeze	\ExpBr{\norm{\nabla f(\hat{x}^T)}^2} \leq \frac{2\Psi^0}{\gamma T}.
	\end{equation}
	where $\Psi^t = f(x^t) - f^{\rm inf} + \frac{\gamma}{A^{\text{M}}}\norm{g^t-\tilde{g}^t}^2 + \frac{\gamma}{A^{\text{W}}}\left(1+\frac{3B^{\text{M}}(2-A^{\text{W}})}{A^{\text{M}}}\right)\avein \|\tilde{g}_i^t - \nabla f_i(x^t)\|^2$ for any $t\geq 0$.
\end{theorem}

In the theorem, superscripts ``M'' and ``W'' denote master and worker compressor parameters, respectively. Theorem~\ref{thm:nonconvex_bidir} implies the following corollary.
\begin{corollary}
	Let the assumptions of~\Cref{thm:nonconvex_bidir} hold, assume $\cM^{\text{M}}$ and $\cM^{\text{W}}$ are \algname{Ada3PC} compressors  and the stepsize 
	$$
	\squeeze	\gamma =  \frac{1}{L_{-} + L_{+}\sqrt{\frac{6B_{\max}^{\text{M}}(B_{\max}^{\text{W}}+1)}{A_{\min}^{\text{M}}} + \frac{2B_{\max}^{\text{W}}}{A_{\min}^{\text{M}}}\left(1+\frac{3B_{\max}^{\text{M}}(2-A_{\min}^{\text{W}})}{A_{\min}^{\text{M}}}\right)}}.
	$$
	Fix $T$ and let $\hat{x}^T$ be chosen uniformly from $\{x^0,x^1,\cdots,x^{T-1}\}$ uniformly at random. 
	Then, we have
	\begin{align*}
		\squeeze		\ExpBr{\norm{\nabla f(\hat{x}^T)}^2} \leq \frac{2\Psi^0\left(L_{-} + L_{+}\sqrt{\frac{6B_{\max}^{\text{M}}(B_{\max}^{\text{W}}+1)}{A_{\min}^{\text{M}}} + \frac{2B_{\max}^{\text{W}}}{A_{\min}^{\text{M}}}\left(1+\frac{3B_{\max}^{\text{M}}(2-A_{\min}^{\text{W}})}{A_{\min}^{\text{M}}}\right)}\right)}{T}.
	\end{align*}
	Thus, to achieve $\ExpBr{\|\nabla f(\hat{x}^T)}\|^2 \leq \varepsilon^2$ for some $\varepsilon > 0$,~\Cref{alg:DCGD_master} requires 
	$$
	\squeeze	T = \cO\left( \frac{2\Psi^0\left(L_{-} + L_{+}\sqrt{\frac{6B_{\max}^{\text{M}}(B_{\max}^{\text{W}}+1)}{A_{\min}^{\text{M}}} + \frac{2B_{\max}^{\text{W}}}{A_{\min}^{\text{M}}}\left(1+\frac{3B_{\max}^{\text{M}}(2-A_{\min}^{\text{W}})}{A_{\min}^{\text{M}}}\right)}\right)}{T}\right)
	$$
	iterations.
\end{corollary}

\section{Experiments}
In this work we use the similar setup described in \citep{3PC}, namely we aim to solve logistic regression problem with non-convex regularizer:
\begin{equation*}
	\squeeze	\min\limits_{x \in \R^d} \left[ f(x) \eqdef  \frac{1}{N}\sum\limits_{i=1}^N \log(1 + e^{-y_i a_i^\top x}) + \lambda \sum\limits_{j=1}^d \frac{x_j^2}{1 + x_j^2}\right],
\end{equation*}
where $a_i \in \R^d$, $y_i \in \{-1, 1\}$ are the training data and labels, and $\lambda > 0$ is a regularization parameter, which is fixed to $\lambda =0.1$. In training we use LIBSVM~\cite{chang2011libsvm} datasets \emph{phishing, a1a, a9a}. Each dataset has been split into $n=20$ equal parts, each representing a different client.

\begin{figure}[H]
	\centering
	\includegraphics[width=\linewidth]{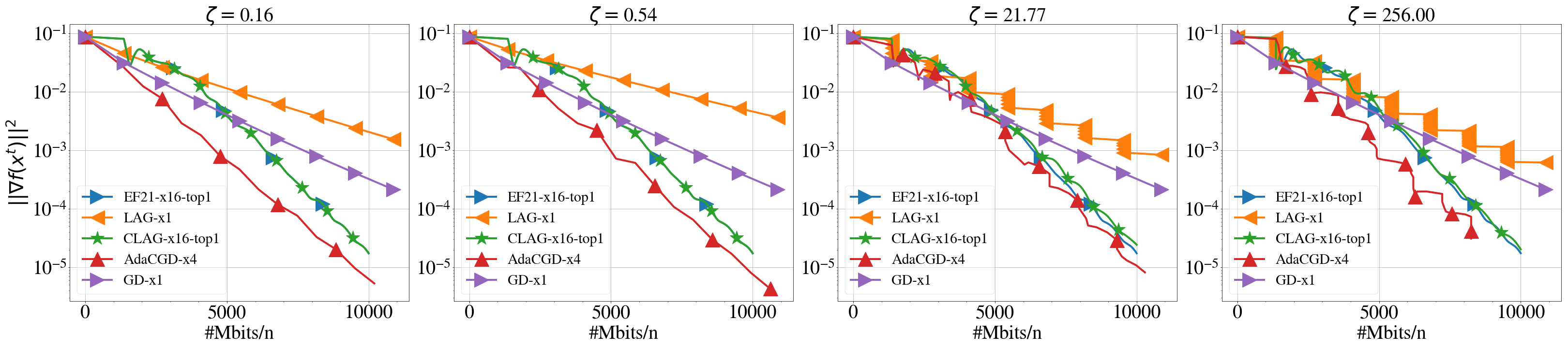}
	\caption{\label{fig:exp} Comparison of \algname{LAG}, \algname{CLAG}, \algname{EF21} and \algname{GD} with \algname{AdaCGD} on phishing dataset. $1\times, 2\times, 4\times$ (and so on) indicates the multiplication factor we use for the optimal stepsizes predicted by theory.}
	\label{fig:don-best}
\end{figure}

\Cref{fig:exp} compares \algname{AdaCGD} with other \algname{3PC} methods. We fine-tune the stepsize of each considered algorithm with $(2^0, 2^1, \dots, 2^8)$ multiples of corresponding theoretical stepsize. As contractive compressor we use Top-$k$ operator. For \algname{EF21} and \algname{CLAG} we use top-1 compressor, which usually the best in practice for these methods. In the experiments, \algname{AdaCGD} is shown to be comparable and in some cases superior to \algname{CLAG} and always superior to \algname{LAG}. In other words, \algname{AdaCGD} efficiently complements \algname{CLAG} and other \algname{3PC} methods. Additional experiments and details are available in the appendix.

\section{Discussion and Limitations}
The main limitations of the work are assumptions we make upon functions $f_i$ of the problem~\ref{eq:finit_sum}. But, on the other hand, these assumptions govern the convergence rates we report: for example, we cannot show linear rate for convex functions due to the fundamental lower bound~\citep{nesterov2018lectures}.

Another limitation comes from the analysis of Bidirectional \algname{3PC} algorithm (\Cref{thm:nonconvex_bidir}). We show the analysis only for general nonconvex functions.

\bibliography{ref}
\bibliographystyle{neurips2022}

\appendix
\clearpage
\part*{APPENDIX}

In~\Cref{sec_ap:basic facts},we state the basic facts needed for detailed proofs of the propositions. In~\Cref{sec_ap: missing proofs}, we provide the proofs missing in the main part of the paper. \Cref{sec_ap:experiments} contains experimental details and extra experiments. We briefly discuss the main limitations of the paper in~\Cref{sec_ap:limitations}.

\tableofcontents

\newpage
\section{Basic facts}\label{sec_ap:basic facts}

We start the appendix with common math facts. Lemmas~\ref{lm:cauchy_schwartz_vec} and~\ref{lm:cauchy_schwartz_prob} present classic Cauchy-Schwartz inequality for vectors in metric space and random variables in probabilistic space, respectively. \Cref{lm:quadratics_expansions} shows a classic upper bound on quadratics. \Cref{lm:aux_ineq} provides a sufficient condition that ensures a quadratic inequality appearing in our convergence proofs holds.

\begin{lemma}[Cauchy-Schwartz inequality for arbitrary vectors]
	\label{lm:cauchy_schwartz_vec}
	Let $x, y \in \RR^d$ be arbitrary vectors. Then, the following inequality holds
	\begin{align}\label{eq:cauchy_schwartz_vec}
		|\la x, y \ra| \leq \|x\| \|y\|,
	\end{align}
	where $\la \cdot, \cdot \ra$ and $\| \cdot \|$ denote the inner product and the induced norm, respectively.
\end{lemma}

\begin{lemma}[Cauchy-Schwartz inequality for random variables; section 6.2.4 of~\citep{pishro2014introduction}]
	\label{lm:cauchy_schwartz_prob}
	For any two random variables $X$ and $Y$, we have
	\begin{align}\label{eq:cauchy_schwartz_prob}
		|\EE [XY]| \leq \sqrt{\EE [X^2] \EE [Y^2]},
	\end{align}
	where equality holds if and only if $X = \alpha Y$, for some constant $\alpha\in\RR$.
\end{lemma}

\begin{lemma}\label{lm:quadratics_expansions}
	Let $a, b, c, d \in \RR^d$ be arbitrary vectors. Then, the following inequalities hold
	\begin{equation}\label{ineq:double_expansion}
		\|a - b\|^2 \leq 2 (\|a - c\|^2 + \|c - b\|^2),
	\end{equation}
	\begin{equation}\label{ineq:triple_expansion}
		\|a - b\|^2 \leq 3 (\|a - c\|^2 + \|c - d\|^2 + \|d - b\|^2).
	\end{equation}
\end{lemma}

\begin{lemma}[Lemma 5 of \citep{ef21}]\label{lm:aux_ineq}
	If $0 < \gamma \leq \frac{1}{\sqrt{a} +b}$, then $a \gamma^2 + b\gamma \leq 1$. Moreover, the bound is tight up to the factor of 2 since $\frac{1}{\sqrt{a} + b} \leq \min\{\frac{1}{\sqrt{a}}, \frac{1}{b} \} \leq \frac{2}{\sqrt{a} + b}$. 
\end{lemma}

\newpage
\section{Proofs for Sections~\ref{sec:ada3pc} and~\ref{sec:theory}}\label{sec_ap: missing proofs}

\subsection{\Cref{lm:adacgd_is_ada3pc}}

At first glance, \algname{AdaCGD} does not seem to be an \algname{Ada3PC} compressor. However, we can construct an \algname{Ada3PC} compressor, which is equivalent to \algname{AdaCGD}.
\begin{customlemma}{\ref{lm:adacgd_is_ada3pc}}
	\algname{AdaCGD} is a special case of \algname{Ada3PC} compressor. 
\end{customlemma}

\begin{proof}
	Let us consider the following~\algname{Ada3PC} compressor constructed from one~\algname{LAG} compressor and $m$ \algname{EF21} compressors. 
	\begin{align*}
		\cC_{h, y}(x) = \begin{cases}
			\cC^{\text{LAG}}_{h,y} = \begin{cases}
				h & \text{ if } \|x - h\|^2 \leq \zeta \|x - y\|^2,\\
				x & \text{ otherwise}.
			\end{cases} &\text{ if } \|x - h\|^2 \leq \zeta \|x - y\|^2,\\
			\cC^{\operatorname{EF},1}_{h,y}(x) &\text{ else if } \|x - \cC^{\operatorname{EF},1}_{h,y}(x)\|^2 \leq \zeta \|x - y\|^2, \\
			\dots, \\
			\cC^{\operatorname{EF},m-1}_{h,y}(x) &\text{ else if } \|x - \cC^{\operatorname{EF},m-1}_{h,y}(x)\|^2 \leq \zeta \|x - y\|^2, \\
			\cC^{\operatorname{EF},m}_{h,y}(x) &\text{ otherwise.}
		\end{cases}
	\end{align*} 
	If $ \|x - h\|^2 \leq \zeta \|x - y\|^2$, then $\cC_{h,y}$ applies the \algname{LAG} compressor to $x$. This \algname{LAG} compressor in turn outputs $h$, as it does $\cC^{\text{AC}}_{h,y}$ for the same condition. If the opposite is true, \textit{i.e.},  $\|x - h\|^2 > \zeta \|x - y\|^2$, $\cC_{h, y}$ checks the same conditions and chooses the same compressor as $\cC^{\text{AC}}_{h,y}$. Thus, $\cC^\text{AC}_{h, y}$ is equivalent to \algname{Ada3PC} compressor $\cC_{h, y}$.
\end{proof}

\subsection{\Cref{thm:3PC_cvx}}

The proof of \Cref{thm:3PC_cvx} requires several auxiliary results. \Cref{lm:descent_lemma} states the descent lemma typical for the analysis of biased compressors. \Cref{tech_lemma} shows how individual \algname{3PC} compressors, applied at clients, affect the aggregated divergence of gradient estimates from gradients. \Cref{lm:lyapunov_product_upperbound} presents a technical upper bound on Lyapunov function~$\Psi^t$.

\begin{lemma}[Lemma 2 of~\citep{PAGE2021}]
	\label{lm:descent_lemma}
	Suppose the function $f$ is $L_{-}$-smooth and $x^{t+1} = x^t - \gamma g^t$, where $g^t \in \R^d$ is any vector, and $\gamma >0$ is any scalar. Then we have 
	\begin{equation}
		\label{eq:descent_lemma}
		f(x^{t+1}) - f(x^t) \leq -\frac{\gamma}{2}\|\nabla f(x^t)\|^2 - \left(\frac{1}{2\gamma} - \frac{L_{-}}{2}\right)\|x^{t+1}-x^t\|^2 + \frac{\gamma}{2}\|g^t - \nabla f(x^t)\|^2.
	\end{equation}
\end{lemma}

\begin{lemma}[Lemma B.3 of~\citep{3PC}]
	\label{tech_lemma}
	Let Assumption \ref{as:lip_avr} hold. Consider \Cref{alg:DCGD_master} with \algname{3PC} compressor $\cM^{\text{W}}$ and identity compressor $\cM^{\text{M}}$.
	Then for all $t\geq 0$ the sequence 
	\begin{equation}
		G^t = \frac{1}{n}\sum^n_{i = 1}\|\nabla f_i(x^t) - g^t_i\|^2
	\end{equation}
	satisfies 
	\begin{equation}
		\label{eq:G_contraction}
		\EE\left[G^{t+1}\right] \leq (1-A)\EE\left[G^t\right] + BL^2_{+}\EE\left[\|x^{t+1} - x^t\|^2\right],
	\end{equation}
	where $A$ and $B$ are parameters of $\cM^{\text{W}}$.
\end{lemma}

\begin{lemma}
	\label{lm:lyapunov_product_upperbound}
	Let Assumption \ref{as:convex} hold. Let Lyapunov function $\Psi^t \eqdef f(x^t) - f^\ast + \frac{\gamma}{A} G^t$. Then, for any $t \geq 0$, the following inequality holds
	\begin{equation}
		\label{eq:lyapunov_product_upperbound}
		\EE\Psi^{t} \leq \sqrt{\left(\EE\left[\|\nabla f(x^t)\|^2 \right] + \frac{\gamma}{A}\EE G^t\right) \left(\EE\left[\|x^t - x^{\star}\|^2\right]  + \frac{\gamma}{A}\EE\left[ G^t\right]\right)},
	\end{equation}
	where $x^\ast$ is any point belonging to $\Argmin{f(x)}$.
\end{lemma}

\begin{proof}
	By definition of convexity we get
	\begin{eqnarray*}
		\EE\Psi^{t} & = & \EE f(x^t) - f^\ast  + \frac{\gamma}{A} \EE G^t \\
		& \overset{\eqref{eq: convexity}}{\leq}& \EE\la\nabla f(x^t),x^t - x^{\star} \ra + \frac{\gamma}{A}\EE G^t\\
		&=&\EE \left\la \left[\nabla f(x^t), \sqrt{\frac{\gamma}{A}\EE G^t}\right], \left[x^t - x^\ast, \sqrt{\frac{\gamma}{A}\EE G^t}\right]\right\ra\\ .
	\end{eqnarray*}
	By applying Cauchy-Schwartz inequality on vectors and random variables we finish the proof
	\begin{align*}
		&\EE \left\la \left[\nabla f(x^t), \sqrt{\frac{\gamma}{A}\EE G^t}\right], \left[x^t - x^\ast, \sqrt{\frac{\gamma}{A}\EE G^t}\right]\right\ra\\
		&\overset{\eqref{eq:cauchy_schwartz_vec}}{\leq} \EE\left[ \sqrt{\|\nabla f(x^t)\|^2  + \frac{\gamma}{A}\EE G^t} \sqrt{\|x^t - x^{\star}\|^2 + \frac{\gamma}{A}\EE G^t }\right]\\
		&\overset{\eqref{eq:cauchy_schwartz_prob}}{\leq} \sqrt{\left(\EE\left[\|\nabla f(x^t)\|^2 \right] + \frac{\gamma}{A}\EE G^t\right) \left(\EE\left[\|x^t - x^{\star}\|^2\right]  + \frac{\gamma}{A}\EE\left[ G^t\right]\right)}.
	\end{align*}
\end{proof}

Now we are ready to prove the main theorem.

\begin{customthm}{\ref{thm:3PC_cvx}}
	Let Assumptions~\ref{as:convex},~\ref{as:lip_f}, ~\ref{as:lip_avr} and~\ref{as:bounded_iterates} hold. Assume the stepsize $\gamma$ of algorithm satisfies $0 < \gamma \leq 1/M$, where $M = L_{-} + L_{+}\sqrt{\frac{2B}{A}}$. Then, for any $T \geq 0 $ we have
	\begin{align*}
		\EE\left[f(x^T)\right] - f(x^{\star}) \leq  \max\left\{\frac{1}{\gamma}, \frac{1}{A}\right\}\frac{2(\Omega^2 +\Psi^0)}{T}.	
	\end{align*}
\end{customthm}

\begin{proof}
	Combining~\Cref{lm:descent_lemma}, Jensen's inequality , we get 
	\begin{equation}\label{eq:aux_descent_lemma}
		\begin{aligned}
			f(x^{t+1}) - f(x^t) &\leq -\frac{\gamma}{2}\|\nabla f(x^t)\|^2 - \left(\frac{1}{2\gamma} - \frac{L_{-}}{2}\right)\|x^{t+1}-x^t\|^2 + \frac{\gamma}{2}\left\|\avein g_i^t - \avein \nabla f_i(x^t)\right\|^2\\
			& \leq -\frac{\gamma}{2}\|\nabla f(x^t)\|^2 - \left(\frac{1}{2\gamma} - \frac{L_{-}}{2}\right)\|x^{t+1}-x^t\|^2 + \frac{\gamma}{2}\avein \|g_i^t - \nabla f_i(x^t)\|^2\\
			& = -\frac{\gamma}{2}\|\nabla f(x^t)\|^2 - \left(\frac{1}{2\gamma} - \frac{L_{-}}{2}\right)\|x^{t+1}-x^t\|^2 + \frac{\gamma}{2} G^t.
		\end{aligned}
	\end{equation}
	Now applying \Cref{eq:aux_descent_lemma} and \Cref{tech_lemma} on the difference of Lyapunov functions, we obtain
	\begin{eqnarray*}
		\EE\left[\Psi^{t+1}\right] - \EE\left[\Psi^t\right] 
		&=& \EE\left[f(x^{t+1}) - f(x^t)\right] +\frac{\gamma}{A}\EE\left[G^{t+1}\right] - \frac{\gamma}{A}\EE\left[G^{t}\right]\\
		&\overset{\eqref{eq:aux_descent_lemma}}{\leq}& -\frac{\gamma}{2}\EE\left[\|\nabla f(x^t)\|^2\right] - \left(\frac{1}{2\gamma} - \frac{L_{-}}{2}\right)\EE\left[\|x^{t+1}-x^t\|^2\right] + \frac{\gamma}{2}\EE\left[G^t\right]\\
		&&\qquad + \frac{\gamma}{A}\EE\left[G^{t+1}\right] - \frac{\gamma}{A}\EE\left[G^{t}\right]\\	
		& \overset{\eqref{eq:G_contraction}}{\leq} & -\frac{\gamma}{2}\EE\left[\|\nabla f(x^t)\|^2\right] - \left(\frac{1}{2\gamma} - \frac{L_{-}}{2}\right)\EE\left[\|x^{t+1}-x^t\|^2\right]\\
		&& \qquad + \frac{\gamma}{A} \left[(1 - A) \EE [G^t] + BL_{+}^2 \EE\|x^{t+1} - x^t\|^2 - \EE [G^t] \right].\\
	\end{eqnarray*}
	Rearranging the term, we get 
	\begin{eqnarray*}
		\EE\left[\Psi^{t+1}\right] - \EE\left[\Psi^t\right] &\leq& -\frac{\gamma}{2}\left[\|\nabla f(x^t)\|^2 \right] - \left(\frac{1}{2\gamma} - \frac{L_{-}}{2} - \frac{\gamma BL^2_{+}}{A}\right)\EE\left[\|x^{t+1}-x^t\|^2\right] - \frac{A}{2}\frac{\gamma}{A}\EE\left[G^t\right].
	\end{eqnarray*}
	We further note that 
	\begin{align*}
		\frac{1}{2\gamma} - \frac{L_{-}}{2} - \frac{\gamma BL^2_{+}}{A} \geq 0 \Leftrightarrow L_{+}^2 \frac{2B}{A} \gamma^2 + L_{-} \gamma \leq 1 \overset{\Cref{lm:aux_ineq}}{\Leftarrow} \gamma \leq \frac{1}{L_{-} + L_{+}\sqrt{\frac{2B}{A}}}.
	\end{align*}
	Appropriately chosen stepsize gives
	\begin{equation*}
		\EE\left[\Psi^{t+1}\right] - \EE\left[\Psi^t\right] \leq -\min\left\{\frac{\gamma}{2}, \frac{A}{2}\right\}\left(\EE\left[\|\nabla f(x^t)\|^2\right]  + \frac{\gamma}{A}\EE\left[G^t\right]\right).
	\end{equation*}
	Rearranging the terms, we have 
	\begin{equation}
		\label{eq:support_cvx}
		\EE\left[\|\nabla f(x^t)\|^2\right]  + \frac{\gamma}{A}\EE\left[G^t\right] \leq \frac{\EE\left[\Psi^t\right] - \EE\left[\Psi^{t+1}\right]}{\min\left\{\frac{\gamma}{2}, \frac{A}{2}\right\}}.
	\end{equation}
	from what we deduce that $\EE\left[\Psi^{t+1}\right] \leq \EE\left[\Psi^t\right]$.

	Using Lemma \ref{lm:lyapunov_product_upperbound} and \eqref{eq:support_cvx}, we have 
	\begin{eqnarray*}
		\EE\Psi^{t+1}\EE\Psi^{t}&\leq& \left(\EE\Psi^{t}\right)^2 \leq \left(\EE\left[\|\nabla f(x^t)\|^2 \right] + \frac{\gamma}{A}\EE G^t\right) \left(\EE\left[\|x^t - x^{\star}\|^2\right]  + \frac{\gamma}{A}\EE\left[ G^t\right]\right)\\
		&\leq& \frac{\EE\left[\|x^t - x^{\star}\|^2\right]  + \frac{\gamma}{A}\EE\left[ G^t\right]}{\min\left\{\frac{\gamma}{2}, \frac{A}{2}\right\}} \left(\EE\left[\Psi^t\right] - \EE\left[\Psi^{t+1}\right]\right).
	\end{eqnarray*}
	
	Using that $\frac{\gamma}{A}\EE\left[ G^t\right] \leq \EE \Psi^t \leq \Psi_0$ and $\EE\left[\|x^t - x^{\star}\|^2\right]  \leq \Omega^2$, we obtain
	\begin{equation*}
		\EE\Psi^{t+1}\EE\Psi^{t}\leq \frac{\Omega^2  + \Psi^0}{\min\left\{\frac{\gamma}{2}, \frac{A}{2}\right\}} \left(\EE\left[\Psi^t\right] - \EE\left[\Psi^{t+1}\right]\right).
	\end{equation*}
	
	Rearranging again, we get
	\begin{equation*}
		\frac{\min\left\{\frac{\gamma}{2}, \frac{A}{2}\right\}}{\Omega^2  + \Psi^0} \leq  \left(\frac{1}{\EE\left[\Psi^{t+1}\right]} - \frac{1}{\EE\left[\Psi^t\right]}\right).
	\end{equation*}
	Summing up from $t = 0$ to $t = T-1$, we finish the proof 
	\begin{equation}
		\EE\left[f(x^T)\right] - f(x^{\star}) \leq\EE\left[\Psi^{T}\right] \leq \max\left\{\frac{2}{\gamma}, \frac{2}{A}\right\}\frac{\Omega^2 +\Psi^0}{T}.
	\end{equation}
\end{proof}

\subsection{\Cref{thm:nonconvex_bidir}}

\begin{algorithm}[H]
	\caption{\algname{3PC-BD}(Bidirectional 3PC algorithm)}
	\label{alg:3PC_EF21_BC}
	\begin{algorithmic}[1]
		\STATE \textbf{Input:} starting point $x^0\in\R^d$; $g^0, \tilde{g}_i^0 \in \R^d$ for $i=1,\cdots, n$ (known by nodes), $\tilde{g}^0 = \frac{1}{n}\sum\limits_{i=1}^n\tilde{g}_i^0$ (known by master); learning rate $\gamma > 0$.
		\FOR{$t=0,1,2,\cdots, T-1$}
		\STATE Broadcast $g^t$ to all workers
		\FOR{{\bf for all devices} $i = 1, \dots, n$ {\bf in parallel}} 
		\STATE $x^{t+1} = x^t - \gamma g^t$
		\STATE $\tilde{g}_i^{t+1} = \cC^w_{\tilde{g}_i^t, \nabla f_i(x^t)}(\nabla f_i(x^{t+1}))$
		\STATE Communicate $\tilde{g}_i^{t+1}$ to the server
		\ENDFOR
		\STATE  $\tilde{g}^{t+1} = \frac{1}{n}\sum\limits_{i=1}^n\tilde{g}_i^{t+1}$
		\STATE  $g^{t+1} = \cC^M_{g^t, \tilde{g}^t}(\tilde{g}^{t+1})$
		\ENDFOR
	\end{algorithmic}
\end{algorithm}

For~\Cref{thm:nonconvex_bidir}, we assume that both compressors $\cM^{\text{W}} $ and $\cM^{\text{M}}$ in \Cref{alg:DCGD_master} are \algname{3PC} compressors. The main steps of the algorithm are:
\begin{align*}
	&x^{t+1} = x^t - \gamma g^t\\
	&\tilde{g}_i^{t+1} = \cC^w_{\tilde{g}_i^t, \nabla f_i(x^t)}(\nabla f_i(x^{t+1}))\\
	&\tilde{g}^{t+1} = \frac{1}{n}\sum\limits_{i=1}^n\tilde{g}_i^{t+1}\\
	&g^{t+1} = \cC^M_{g^t, \tilde{g}^t}(\tilde{g}^{t+1})
\end{align*}

Unlike in the previous subsection, we use additional notations: $P_i^t~\eqdef~\norm{\tilde{g}_i^t - \nabla f_i(x^t)}^2, P^t~\eqdef~\frac{1}{n}\sum\limits_{i=1}^nP_i^t$ and $R^t \eqdef \norm{x^{t+1}-x^t}^2$.

\Cref{lm:BD_aux1} is an analogue of~\Cref{tech_lemma} (in bidirectional case we need slightly different arguments at some steps). \Cref{lm:BD_aux2} is another technical lemma that upper bounds $\ExpBr{\|g^t - \tilde{g}^t}\|^2$. 

\begin{lemma}\label{lm:BD_aux1}
	Let Assumption~\ref{as:lip_avr} hold, $\cC^w$ be a 3PC compressor, and $\tilde{g}_i^{t+1}$ be an \algname{3PC-BD} estimator of $\nabla f_i(x^{t+1}), i.e.$
	\begin{equation}
		\tilde{g}_i^{t+1} = \cC^w_{\tilde{g}_i^t,\nabla f_i(x^t)}(\nabla f_i(x^{t+1}))
	\end{equation}
	for arbitrary $\tilde{g}_i^0$ for all $i\in [n], t\geq 0.$ Then
	\begin{equation}\label{eq:ineq_for_Pt}
		\ExpBr{P^{t+1}}\leq (1-A^{\text{W}})\ExpBr{P^t} + B^{\text{W}} L_{+}^2\ExpBr{R^t}.
	\end{equation}
\end{lemma}

\begin{proof}
	Define $W^t \eqdef \{\tilde{g}_1^t, \cdots, \tilde{g}_n^t, x^t,x^{t+1}\}$, then 
	\begin{equation}
		\label{eq:aux_BD_1}
		\begin{aligned}
			\ExpBr{P_i^{t+1}} &= \ExpBr{\ExpBr{P_i^{t+1} \;|\; W^t}}\\
			&= \ExpBr{\ExpBr{\norm{\tilde{g}_i^{t+1} - \nabla f_i(x^{t+1})}^2 \;|\; W^t}}\\
			&= \ExpBr{\ExpBr{\norm{\cC^w_{\tilde{g}_i^t,\nabla f_i(x^t)}(\nabla f_i(x^{t+1})) - \nabla f_i(x^{t+1})}^2 \;|\; W^t}}\\
			&\overset{\eqref{ineq:3PC_def}}{\leq} (1-A^{\text{W}})\ExpBr{\norm{\tilde{g}_i^t - \nabla f_i(x^t)}^2} + B^{\text{W}}\ExpBr{\norm{\nabla f_i(x^{t+1}) - \nabla f_i(x^t)}^2}\\.
		\end{aligned}
	\end{equation}
	Averaging the above inequalities over $i\in[n]$, we obtain~\eqref{eq:ineq_for_Pt}. Indeed,
	\begin{align*}
		\ExpBr{P^{t+1}} &= \ExpBr{\avein P_i^{t+1}} = \avein \ExpBr{P_i^{t+1}} \\
		& \overset{\eqref{eq:aux_BD_1}}{\leq} \avein  (1-A^{\text{W}})\ExpBr{\norm{\tilde{g}_i^t - \nabla f_i(x^t)}^2} + \avein B^{\text{W}}\ExpBr{\norm{\nabla f_i(x^{t+1}) - \nabla f_i(x^t)}^2}\\
		& = (1 - A^{\text{W}}) \ExpBr{P^t} +  B^{\text{W}}\avein\ExpBr{\norm{\nabla f_i(x^{t+1}) - \nabla f_i(x^t)}^2}\\
		& \overset{\Cref{as:lip_avr}}{\leq} (1 - A^{\text{W}}) \ExpBr{P^t} + B^{\text{W}} L_{+}^2 \EE \|x^{t+1} - x^t\|^2\\
		& = (1 - A^{\text{W}}) \ExpBr{P^t} + B^{\text{W}} L_{+}^2 \ExpBr{R^t}.\\
	\end{align*}
\end{proof}

\begin{lemma}\label{lm:BD_aux2}
	Let Assumptions~\ref{as:lip_avr} and~\ref{as:lower_bound} hold, $\cC^M, \cC^w$ be 3PC compressors. Let $\tilde{g}_i^{t+1}$ be an \algname{3PC-BD} estimator of $\nabla f_i(x^{t+1})$, i.e.
	\begin{equation}
		\tilde{g}_i^{t+1} = \cC^w_{\tilde{g}_i^t,\nabla f_i(x^t)}(\nabla f_i(x^{t+1}))
	\end{equation}
	and let $g^{t+1}$ be an \algname{3PC-BD} estimator of $\tilde{g}^{t+1}=\frac{1}{n}\sum\limits_{i=1}^n\tilde{g}_i^{t+1}$, i.e.
	\begin{equation}
		g_i^{t+1} = \cC^M_{g^t,\tilde{g}^t}(\tilde{g}^{t+1})
	\end{equation}
	for arbitrary $g^0, \tilde{g}_i^0$ for all $i\in [n], t\geq 0.$ Then
	\begin{equation}\label{eq:ineq_for_gt_gt_tilde}
		\ExpBr{\norm{g^{t+1} - \tilde{g}^{t+1}}^2} \leq (1-A^{\text{M}})\ExpBr{\norm{g^t-\tilde{g}^t}^2} + 3B^{\text{M}}(2-A^{\text{W}})\ExpBr{P^t} +  3B^{\text{M}}(B^{\text{W}}+1)L_{+}^2\ExpBr{R^t},
	\end{equation}
	where $g^t = \frac{1}{n}\sum_{i=1}^ng_i^t, \tilde{g}^t = \frac{1}{n}\sum_{i=1}^n\tilde{g}_i^t.$
\end{lemma}
\begin{proof}
	Similarly to the proof of Lemma~\ref{lm:BD_aux1}, we define $W^t \eqdef \{g_1^t, \cdots, g_n^t, x^t,x^{t+1}\}$ and bound $\ExpBr{\norm{g^{t+1} -\tilde{g}^{t+1}}^2} $:
	\begin{align}
		\ExpBr{\norm{g^{t+1} -\tilde{g}^{t+1}}^2} &= \ExpBr{\ExpBr{ \norm{g^{t+1}-\tilde{g}^{t+1}}^2\;|\; W^t}}\notag\\
		&= \ExpBr{\ExpBr{ \norm{\cC^M_{g^t,\tilde{g}^t}(\tilde{g}^{t+1})-\tilde{g}^{t+1}}^2\;|\; W^t}}\notag\\
		&\overset{\eqref{ineq:3PC_def}}{\leq} (1-A^{\text{M}})\ExpBr{\norm{g^{t} -\tilde{g}^{t}}^2} + B^{\text{M}}\ExpBr{\norm{\tilde{g}^{t+1}-\tilde{g}^t}^2}\label{eq:ineq_for_gt_gt_tilde_first},
	\end{align}
	Further, we bound the last term in~\eqref{eq:ineq_for_gt_gt_tilde_first}. Recall that
	\begin{equation}
		\tilde{g}^{t+1} = \frac{1}{n}\sum\limits_{i=1}^n\tilde{g}^{t+1}_i = \frac{1}{n}\sum\limits_{i=1}^n\cC^w_{\tilde{g}_i^t, \nabla f_i(x^t)}(\nabla f_i(x^{t+1})).
	\end{equation}
	Then,
	\begin{align}
		\ExpBr{\norm{\tilde{g}^{t+1}-\tilde{g}^t}^2} &= \ExpBr{\norm{\frac{1}{n}\sum\limits_{i=1}^n\cC^w_{\tilde{g}_i^t, \nabla f_i(x^t)}(\nabla f_i(x^{t+1})) - \tilde{g}^t_i}^2}\notag\\
		&\leq \frac{1}{n}\sum\limits_{i=1}^n\ExpBr{\norm{\cC^w_{\tilde{g}_i^t, \nabla f_i(x^t)}(\nabla f_i(x^{t+1})) - \tilde{g}^t_i}^2}\notag\\
		&\overset{\eqref{ineq:triple_expansion}}{\leq} \frac{3}{n}\sum\limits_{i=1}^n\ExpBr{\norm{\cC^w_{\tilde{g}_i^t, \nabla f_i(x^t)}(\nabla f_i(x^{t+1})) - \nabla f_i(x^{t+1})}^2}\notag\\
		&\quad + \frac{3}{n}\sum\limits_{i=1}^n\ExpBr{\norm{\nabla f_i(x^{t+1}) - \nabla f_i(x^{t})}^2}+ \frac{3}{n}\sum\limits_{i=1}^n\ExpBr{\norm{\nabla f_i(x^{t}) - \tilde{g}_i^t}^2}\notag\\
		&\overset{\eqref{ineq:3PC_def}}{\leq} 3(1-A^{\text{W}})\frac{1}{n}\sum\limits_{i=1}^n\ExpBr{\norm{\nabla f_i(x^{t}) - \tilde{g}_i^t}^2} + 3B^{\text{W}}\frac{1}{n}\sum\limits_{i=1}^n\ExpBr{\norm{\nabla f_i(x^{t+1}) - \nabla f_i(x^{t})}^2}\notag\\
		&\quad + \frac{3}{n}\sum\limits_{i=1}^n\ExpBr{\norm{\nabla f_i(x^{t+1}) - \nabla f_i(x^{t})}^2} + \frac{3}{n}\sum\limits_{i=1}^n\ExpBr{\norm{\nabla f_i(x^{t}) - \tilde{g}_i^t}^2}\notag\\
		&\overset{\Cref{as:lip_avr}}{\leq} 3(2-A^{\text{W}})\ExpBr{P^t} + 3(B^{\text{W}}+1)L_{+}^2\ExpBr{\norm{x^{t+1}-x^t}^2}\notag\\
		&= 3(2-A^{\text{W}})\ExpBr{P^t} + 3(B^{\text{W}}+1)L_{+}^2\ExpBr{R^t},\label{eq:ineq_for_gt_tilde_gt_tilde}
	\end{align}
	where the first inequality follows from Young's inequality. Plugginq~\eqref{eq:ineq_for_gt_tilde_gt_tilde} into~\eqref{eq:ineq_for_gt_gt_tilde_first} we finish the proof:
	\begin{align*}
		\ExpBr{\norm{g^{t+1} -\tilde{g}^{t+1}}^2} 
		&\leq (1-A^{\text{M}})\ExpBr{\norm{g^{t} -\tilde{g}^{t}}^2} + 3B^{\text{M}}(2-A^{\text{W}})\ExpBr{P^t}\\
		&\quad + 3B^{\text{M}}(B^{\text{W}}+1)L_{+}^2\ExpBr{R^t}.
	\end{align*}
\end{proof}

Having proved the previous lemmas, we can now show the convergence of bidirectional \algname{3PC} algorithm.

\begin{customthm}{\ref{thm:nonconvex_bidir}}
	Let Assumptions~\ref{as:lip_avr} and~\ref{as:lower_bound} hold, and let the stepsize in Algorithm~\ref{alg:3PC_EF21_BC} be set as
	\begin{equation}
		0 \leq \gamma < \left(L_{-} + L_{+}\sqrt{\frac{6B^{\text{M}}(B^{\text{W}}+1)}{A^{\text{M}}} + \frac{2B^{\text{W}}}{A^{\text{M}}}\left(1+\frac{3B^{\text{M}}(2-A^{\text{W}})}{A^{\text{M}}}\right)}\right)^{-1}.
	\end{equation}
	Fix $T$ and let $\hat{x}^T$ be chosen uniformly from $\{x^0,x^1,\cdots,x^{T-1}\}$ uniformly at random. Then 
	\begin{equation}
		\ExpBr{\norm{\nabla f(\hat{x}^T)}^2} \leq \frac{2\Psi^0}{\gamma T}.
	\end{equation}
	where $\Psi^T = f(x^t) - f^{\rm inf} + \frac{\gamma}{A^{\text{M}}}\norm{g^t-\tilde{g}^t}^2 + \frac{\gamma}{A^{\text{W}}}\left(1+\frac{3B^{\text{M}}(2-A^{\text{W}})}{A^{\text{M}}}\right)\avein \|\tilde{g}_i^t - \nabla f_i(x^t)\|^2$.
\end{customthm}

\begin{proof}
	We apply Lemma~\ref{lm:descent_lemma} and split the error $\norm{g^t-\nabla f(x^t)}^2$ into two parts
	\begin{align}
		f(x^{t+1}) &\overset{\eqref{eq:descent_lemma}}{\leq} f(x^t) - \frac{\gamma}{2}\norm{\nabla f(x)}^2 -\left(\frac{1}{2\gamma}-\frac{L_{-}}{2}\right)R^t + \frac{\gamma}{2}\norm{g^t-\nabla f(x^t)}^2\notag\\
		&\overset{\eqref{ineq:double_expansion}}{\leq} f(x^t) - \frac{\gamma}{2}\norm{\nabla f(x)}^2 -\left(\frac{1}{2\gamma}-\frac{L_{-}}{2}\right)R^t + \gamma\norm{\tilde{g}^t-\nabla f(x^t)}^2 + \gamma\norm{\tilde{g}^t-g^t}^2\notag\\
		&\leq f(x^t) - \frac{\gamma}{2}\norm{\nabla f(x)}^2 -\left(\frac{1}{2\gamma}-\frac{L_{-}}{2}\right)R^t + \frac{\gamma}{n}\sum\limits_{i=1}^n\norm{\tilde{g}_i^t-\nabla f_i(x^t)}^2 + \gamma\norm{\tilde{g}^t-g^t}^2\notag\\
		&= f(x^t) - \frac{\gamma}{2}\norm{\nabla f(x)}^2 -\left(\frac{1}{2\gamma}-\frac{L_{-}}{2}\right)R^t + \gamma P^t + \gamma\norm{\tilde{g}^t-g^t}^2\label{eq11},
	\end{align}
	where in the last inequality we applied Young's inequality.
	Subtracting $f^{\rm inf}$ from both sides of the above inequality, taking expectation and using the notation $\delta^t = f(x^t) - f^{\rm inf},$ we get
	\begin{equation}\label{eq12}
		\ExpBr{\delta^{t+1}} \leq \ExpBr{\delta^t} - \frac{\gamma}{2}\ExpBr{\norm{\nabla f(x^t)}^2} - \left(\frac{1}{2\gamma}-\frac{L_{-}}{2}\right)\ExpBr{R^t} + \gamma \ExpBr{P^t} + \gamma\ExpBr{\norm{\tilde{g}^t-g^t}^2}.
	\end{equation}
	Further, Lemmas~\ref{lm:BD_aux1} and~\ref{lm:BD_aux2} provide the recursive bounds for the last two terms of \eqref{eq12}
	\begin{align}
		\ExpBr{P^{t+1}} &\leq (1-A^{\text{W}})\ExpBr{P^t} + B^{\text{W}} L_{+}^2\ExpBr{R^t},\label{eq13}\\
		\ExpBr{\norm{g^{t+1} - \tilde{g}^{t+1}}^2} &\leq (1-A^{\text{M}})\ExpBr{\norm{g^t-\tilde{g}^t}^2} + 3B^{\text{M}}(2-A^{\text{W}})\ExpBr{P^t}\notag\\
		&\qquad + 3B^{\text{M}}(B^{\text{W}}+1)L_{+}^2\ExpBr{R^t}\label{eq14}.
	\end{align}
	Summing up~\eqref{eq12} with a $\frac{\gamma}{A^{\text{M}}}$ multiple of~\eqref{eq14} we obtain
	\begin{align*}
		\ExpBr{\delta^{t+1}} + \frac{\gamma}{A^{\text{M}}}\ExpBr{\norm{g^t-\tilde{g}^t}^2} &\leq \ExpBr{\delta^t} - \frac{\gamma}{2}\ExpBr{\norm{\nabla f(x^t)}^2} - \left(\frac{1}{2\gamma}-\frac{L_{-}}{2}\right)\ExpBr{R^t}\\
		&\qquad + \gamma \ExpBr{P^t} + \gamma\ExpBr{\norm{\tilde{g}^t-g^t}^2}\\
		&\qquad +\frac{\gamma}{A^{\text{M}}}\left((1-A^{\text{M}})\ExpBr{\norm{g^t-\tilde{g}^t}^2}\right)\\
		&\qquad +\frac{\gamma}{A^{\text{M}}}\left(3B^{\text{M}}(2-A^{\text{W}})\ExpBr{P^t}+ 3B^{\text{M}}(B^{\text{W}}+1)L_{+}^2\ExpBr{R^t}\right)\\
		&\leq \ExpBr{\delta^t} - \frac{\gamma}{2}\ExpBr{\norm{\nabla f(x^t)}^2} + \frac{\gamma}{A^{\text{M}}}\ExpBr{\norm{g^t-\tilde{g}^t}^2}\\
		&\qquad - \left(\frac{1}{2\gamma}-\frac{L_{-}}{2} -\frac{3\gamma B^{\text{M}}(B^{\text{W}}+1)L_{+}^2}{A^{\text{M}}}\right)\ExpBr{R^t}\\
		&\qquad +\gamma\left(1+\frac{3B^{\text{M}}(2-A^{\text{W}})}{A^{\text{M}}}\right)\ExpBr{P^t}.
	\end{align*}
	Then adding the above inequality with a $\frac{\gamma}{A^{\text{W}}}\left(1+\frac{3B^{\text{M}}(2-A^{\text{W}})}{A^{\text{M}}}\right)$ multiple of~\eqref{eq13}, we get
	\begin{align}
		\ExpBr{\Psi^{t+1}} &= \ExpBr{\delta^{t+1}} + \frac{\gamma}{A^{\text{M}}}\ExpBr{\norm{g^t-\tilde{g}^t}^2} + \frac{\gamma}{A^{\text{W}}}\left(1+\frac{3B^{\text{M}}(2-A^{\text{W}})}{A^{\text{M}}}\right)\ExpBr{P^{t+1}}\notag\\
		&\leq \ExpBr{\delta^t} - \frac{\gamma}{2}\ExpBr{\norm{\nabla f(x^t)}^2} + \frac{\gamma}{A^{\text{M}}}\ExpBr{\norm{g^t-\tilde{g}^t}^2}\notag\\
		&\qquad - \left(\frac{1}{2\gamma}-\frac{L_{-}}{2} -\frac{3\gamma B^{\text{M}}(B^{\text{W}}+1)L_{+}^2}{A^{\text{M}}}\right)\ExpBr{R^t} +\gamma\left(1+\frac{3B^{\text{M}}(2-A^{\text{W}})}{A^{\text{M}}}\right)\ExpBr{P^t}\notag\\
		&\qquad + \frac{\gamma}{A^{\text{W}}}\left(1+\frac{3B^{\text{M}}(2-A^{\text{W}})}{A^{\text{M}}}\right)\left((1-A^{\text{W}})\ExpBr{P^t} + B^{\text{W}}L_{+}^2\ExpBr{R^t}\right)\notag\\
		&\leq \ExpBr{\delta^t}+ \frac{\gamma}{A^{\text{M}}}\ExpBr{\norm{g^t-\tilde{g}^t}^2} + \frac{\gamma}{A^{\text{W}}}\left(1+\frac{3B^{\text{M}}(2-A^{\text{W}})}{A^{\text{M}}}\right)\ExpBr{P^t} - \frac{\gamma}{2}\ExpBr{\norm{\nabla f(x^t)}^2}\notag\\
		&\qquad -\left(\frac{1}{2\gamma}-\frac{L_{-}}{2} -\frac{3\gamma B^{\text{M}}(B^{\text{W}}+1)L_{+}^2}{A^{\text{M}}} - \frac{\gamma B^{\text{W}}L_{+}^2}{A^{\text{W}}}\left(1+\frac{3B^{\text{M}}(2-A^{\text{W}})}{A^{\text{M}}}\right)\right)\ExpBr{R^t}.\label{eq15}
	\end{align}
	Thus by Lemma~\ref{lm:aux_ineq} and the choice of the stepsize
	\begin{equation}\label{eq16}
		0 \leq \gamma < \left(L + L_{+}\sqrt{\frac{6B^{\text{M}}(B^{\text{W}}+1)}{A^{\text{M}}} + \frac{2B^{\text{W}}}{A^{\text{M}}}\left(1+\frac{3B^{\text{M}}(2-A^{\text{W}})}{A^{\text{M}}}\right)}\right)^{-1},
	\end{equation}
	the last term in~\eqref{eq15} is not positive. By summing up inequalities for $t=0, 1,\cdots, T-1,$ we get
	\begin{equation*}
		0 \leq \ExpBr{\Psi^T} \leq \Psi^0 - \frac{\gamma}{2}\sum\limits_{i=1}^{T-1}\ExpBr{\norm{\nabla f(x^t)}^2}.
	\end{equation*}
	Multiplying both sides by $\frac{2}{\gamma T}$ and rearranging we get
	\begin{equation*}
		\frac{1}{T}\sum\limits_{i=1}^{T-1}\ExpBr{\norm{\nabla f(x^t)}^2} \leq \frac{2\Psi^0}{\gamma T}.
	\end{equation*}
\end{proof}

\subsection{Convergence for general nonconvex functions}

The results in two subsequent subsections set $\cM^{\text{W}}$ as a~\algname{3PC} compressor and $\cM^{\text{M}}$ as an indentity one. According to~\Cref{lm:adaptive_3PC_is_3PC}, Adaptive 3PC is a 3PC compressor. Thus, convergence results from~\citep{3PC} are valid for Adaptive 3PC compressor. It leads us to the following corollary.
\begin{corollary}[Corollary 5.6 of~\citep{3PC}]
	Let Assumptions~\ref{as:lip_f},~\ref{as:lip_avr} and~\ref{as:lower_bound} hold. Let $\cM^{\text{W}}$ and $\cM^{\text{M}}$ in~\Cref{alg:DCGD_master} be~\algname{Ada3PC} and identity compressors, respectively,  and choose the stepsize $\gamma = \frac{1}{L_{-} + L_{+} \sqrt{\frac{B_{\max}}{A_{\min}}}}$. Then, for any $T \geq 1$ we have 
	$$
	\ExpBr{\|\nabla f(\hat{x}^T)\|^2} \leq \frac{2(f(x^0) - f(x^{\inf}))\left(L_{-} + L_{+} \sqrt{\frac{B_{\max}}{A_{\min}}}\right)}{T} + \frac{\Exp{\avein \|g_i^0 - \nabla f_i(x^0)\|^2}}{A_{\min}T}.
	$$ 
	That is, to achieve $\ExpBr{\|\nabla f(\hat{x}^T)\|^2} \leq \varepsilon^2$ for some $\varepsilon > 0$,~\Cref{alg:DCGD_master} requires 
	$$
	T = \cO\left(\frac{2(f(x^0) - f(x^{\inf}))\left(L_{-} + L_{+} \sqrt{\frac{B_{\max}}{A_{\min}}}\right)}{\varepsilon^2} + \frac{\Exp{\avein \|g_i^0 - \nabla f_i(x^0)\|^2}}{A_{\min}\varepsilon^2}\right)
	$$
	iterations.
\end{corollary}

\subsection{Convergence for P\L nonconvex functions}

The setup here is the same as in the previous subsection, except we add the following assumption.

\begin{assumption}[P\L \; condition]\label{as:PL}
	Function $f: \RR^d \rightarrow \RR$ satisfies the Polyak-\L ojasiewicz (P\L) condition with parameter $\mu > 0$, i.e.,
	$$
	\|\nabla f(x)\|^2 \geq 2\mu(f(x) - f^\ast) \quad \forall x \in \RR^d,
	$$	
	where $x^\ast \coloneqq \argmin\limits_{x \in \RR^d} f(x)$ and $f^\ast \coloneqq f(x^\ast)$.
\end{assumption}

\begin{corollary}[Corollary 5.9 of~\citep{3PC}]
	Let Assumptions~\ref{as:lip_f},~\ref{as:lip_avr},~\ref{as:lower_bound} and~\ref{as:PL} hold. Let $\cM^{\text{W}}$ and $\cM^{\text{M}}$ in~\Cref{alg:DCGD_master} be~\algname{Ada3PC} and identity compressors, respectively, and choose the stepsize 
	$$
	\gamma = \min\left\{\frac{1}{L_{-} + L_{+}\sqrt{\frac{2B_{\max}}{A_{\min}}}}, \frac{A_{\min}}{2\mu}\right\}.
	$$
	Then, to achieve $\ExpBr{f(x^T)} - f^\ast \leq \varepsilon$ for some $\varepsilon > 0$ the method requires 
	$$
	\cO\left(\max\left\{\frac{L_{-} + L_{+}\sqrt{\frac{B_{\max}}{A_{\min}}}}{\mu}, A_{\min} \right\}\log{\frac{f(x^0) - f(x^{\inf}) + \Exp{\avein \|g_i^0 - \nabla f_i(x^0)\|^2\gamma / A_{\min}}}{\varepsilon}}\right)
	$$
	iterations.
\end{corollary}

\newpage
\section{Experimental details and extra experiments}\label{sec_ap:experiments}

All simulations are implemented in Python 3.8 and run on Intel(R) Xeon(R) Gold 6230R CPU cluster with 48 nodes. We fine-tune the stepsize of each considered algorithm with $(2^0, 2^1, \dots, 2^8)$ multiples of the corresponding theoretical stepsize. As contractive compressor we use Top-$k$ operator. For \algname{EF21} and \algname{CLAG} we use top-1 compressor, which usually the best in practice for these methods. For \algname{AdaCGD} we choose compressors varying from full  compression (skip communication) to compression of 50\% of features. In order to provide fair comparisons, we choose master compressor $\cM^M$ as identity operator in these experiments. For the stopping criterion we choose communication cost of the algorithm. 

We use the setup described in \citet{3PC}, namely logistic regression with non-convex regularizer:

\begin{equation*}
	\squeeze	\min\limits_{x \in \R^d} \left[ f(x) \eqdef  \frac{1}{N}\sum\limits_{i=1}^N \log(1 + e^{-y_i a_i^\top x}) + \lambda \sum\limits_{j=1}^d \frac{x_j^2}{1 + x_j^2}\right],
\end{equation*}
where $a_i \in \R^d$, $y_i \in \{-1, 1\}$ are the training data and labels, and $\lambda > 0$ is a regularization parameter, which is fixed to $\lambda =0.1$. We solve this problem using LIBSVM~\cite{chang2011libsvm}  datasets \emph{phishing, a1a, a9a}. Each dataset has been evenly split into $n=20$ equal parts where each part represents a separate client. Figures \ref{fig:phishing}-\ref{fig:a9a} compare \algname{AdaCGD} with \algname{LAG}, \algname{EF21} and their generalization \algname{CLAG}. In the experiments, \algname{AdaCGD} is shown to be comparable and in some cases superior to \algname{CLAG} and always superior to \algname{LAG}. In other words, \algname{AdaCGD} efficiently complements \algname{CLAG} and other \algname{3PC} methods.

\begin{figure}[H]
	\centering
	\includegraphics[width=\linewidth]{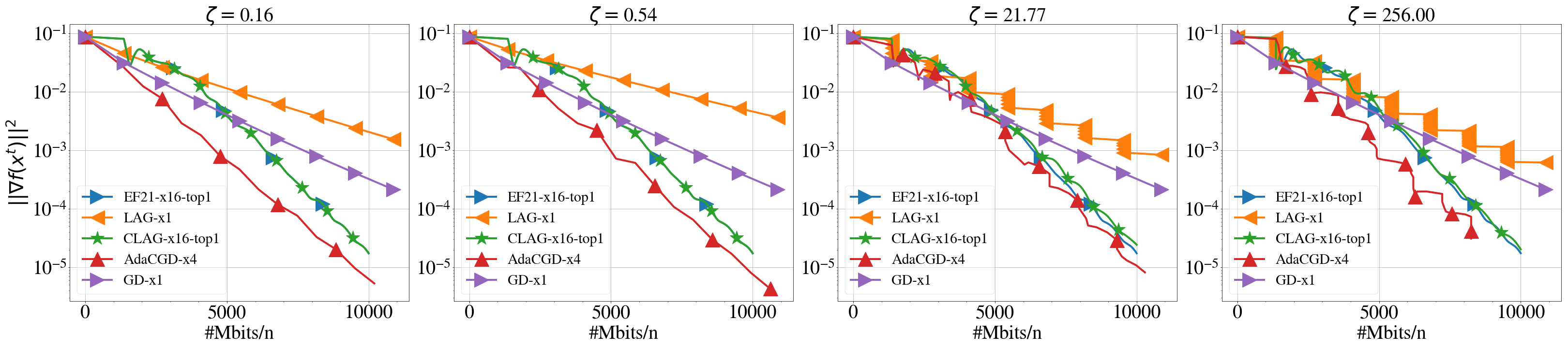}
	\caption{\label{fig:exp_phishing} Comparison of \algname{LAG}, \algname{CLAG}, \algname{EF21} and \algname{GD} with \algname{AdaCGD} on {\em phishing} dataset.}
	\label{fig:phishing}
\end{figure}

\begin{figure}[H]
	\centering
	\includegraphics[width=\linewidth]{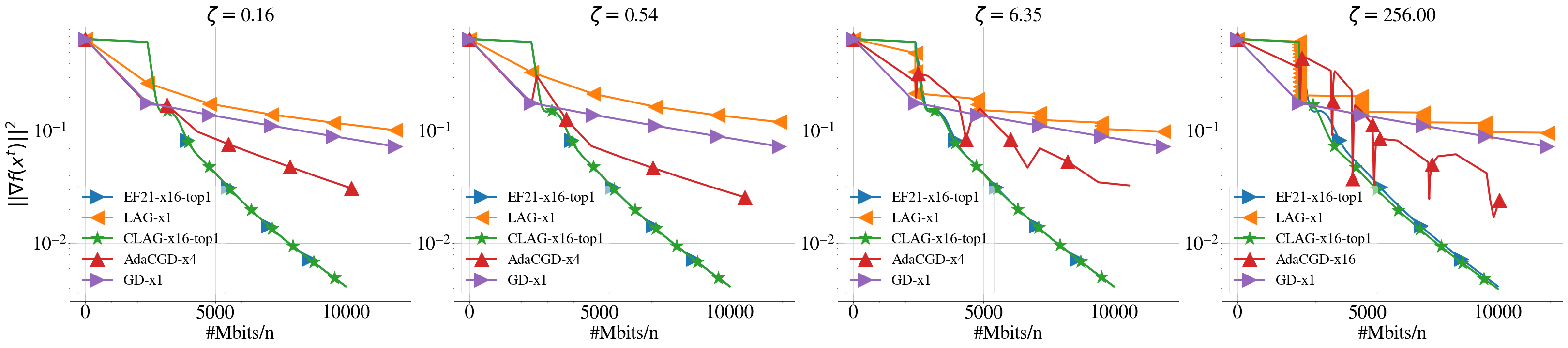}
	\caption{\label{fig:exp_a9a} Comparison of \algname{LAG}, \algname{CLAG}, \algname{EF21} and \algname{GD} with \algname{AdaCGD} on {\em a1a} dataset.}
	\label{fig:a1a}
\end{figure}

\begin{figure}[H]
	\centering
	\includegraphics[width=\linewidth]{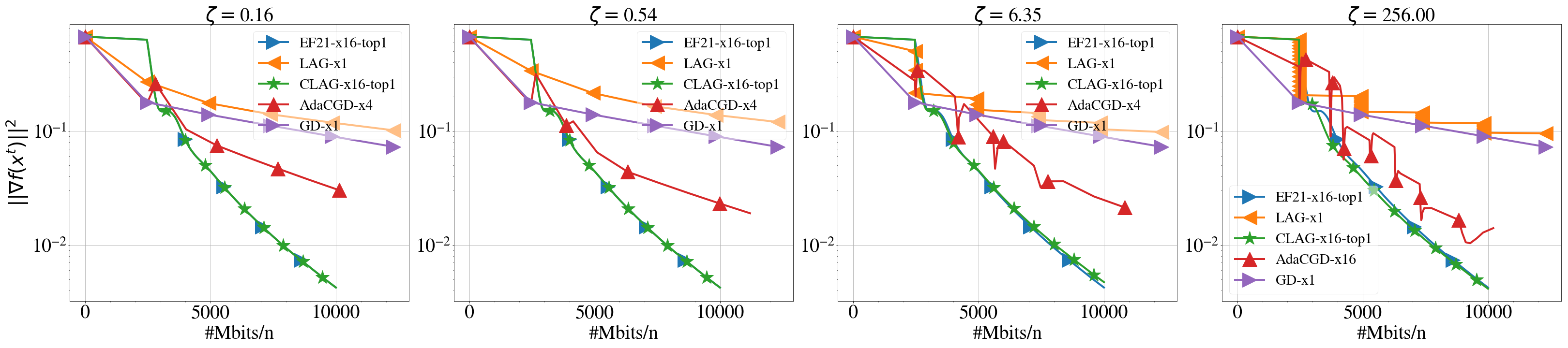}
	\caption{\label{fig:exp_a9a} Comparison of \algname{LAG}, \algname{CLAG}, \algname{EF21} and \algname{GD} with \algname{AdaCGD} on {\em a9a} dataset.}
	\label{fig:a9a}
\end{figure}

\newpage
\section{Limitations}\label{sec_ap:limitations}
The main limitations of the work are assumptions we make upon functions $f_i$ of the problem~\ref{eq:finit_sum}. But, on the other hand, these assumptions govern the convergence rates we report: for example, we cannot show linear rate for convex functions due to the fundamental lower bound~\citep{nesterov2018lectures}.

Another limitation comes from the analysis of Bidirectional \algname{3PC} algorithm (\Cref{thm:nonconvex_bidir}). We show the analysis only for general nonconvex functions.

\end{document}

%% file: new_commands.tex
\newcommand{\ExpBr}[1]{\mathbb{E}\left[#1\right]}
\usepackage{xspace}

\newcommand{\avein}{\frac{1}{n}\sum_{i=1}^n}

\newcommand{\EE}{\mathbb{E}}

\newcommand{\RR}{\mathbb{R}}

\let\la=\langle
\let\ra=\rangle

\newtheorem{innercustomthm}{Theorem}
\newenvironment{customthm}[1]
  {\renewcommand\theinnercustomthm{#1}\innercustomthm}
  {\endinnercustomthm}

\newtheorem{innercustomlemma}{Lemma}
\newenvironment{customlemma}[1]
	{\renewcommand\theinnercustomlemma{#1}\innercustomlemma}
	{\endinnercustomlemma}

%% file: theorem_environments.tex
\newtheorem{assumption}{Assumption}
\newtheorem{lemma}{Lemma}

\newtheorem{theorem}{Theorem}

\newtheorem{corollary}{Corollary}

\theoremstyle{plain}

\theoremstyle{definition}

\newtheorem{definition}[theorem]{Definition}

%% file: new_math_operators.tex
\DeclareMathOperator*{\Argmin}{Argmin}